\documentclass{article}

\usepackage{amsmath,amsfonts,amssymb,times,graphicx,natbib,algorithm,algorithmic,hyperref}

\usepackage{microtype}
\usepackage{graphicx}
\usepackage{subfigure}
\usepackage{booktabs} 
\usepackage{amsmath}
\usepackage{amssymb}
\usepackage{amsthm}
\usepackage{csquotes}
\usepackage{mathtools}
\usepackage{hyperref}

\usepackage{mleftright}
\usepackage{xparse}
\usepackage[resetlabels]{multibib}

\newcites{New}{References}

\NewDocumentCommand{\evalat}{sO{\big}mm}{%
  \IfBooleanTF{#1}
   {\mleft. #3 \mright|_{#4}}
   {#3#2|_{#4}}%
}

\DeclareMathOperator{\argmaxG}{arg\,max} 

\newtheorem{theorem}{Theorem}
\newtheorem{corollary}{Corollary}[theorem]
\newtheorem{lemma}[theorem]{Lemma}
\newtheoremstyle{case}{}{}{}{}{}{:}{ }{}
\theoremstyle{case}

\DeclarePairedDelimiter\ceil{\lceil}{\rceil}
\DeclarePairedDelimiter\floor{\lfloor}{\rfloor}

\usepackage[accepted]{whi2018}

\icmltitlerunning{Noise-adding Methods of Saliency Map as Series of Higher Order Partial Derivative}

\begin{document}

\twocolumn[
\icmltitle{Noise-adding Methods of Saliency Map\\
as Series of Higher Order Partial Derivative}


\icmlsetsymbol{equal}{*}

\begin{icmlauthorlist}
\icmlauthor{Junghoon Seo}{si}
\icmlauthor{Jeongyeol Choe}{si}
\icmlauthor{Jamyoung Koo$^*$}{si}
\icmlauthor{Seunghyeon Jeon$^*$}{si}
\icmlauthor{Beomsu Kim$^*$}{ka}
\icmlauthor{Taegyun Jeon}{si}
\end{icmlauthorlist}

\icmlaffiliation{si}{Satrec Initiative, Daejeon, Rep. of Korea}
\icmlaffiliation{ka}{School of Computing, Korea Advanced Institute of Science and Technology, Daejeon, Rep. of Korea}

\icmlcorrespondingauthor{Junghoon Seo}{sjh@satreci.com}
\icmlcorrespondingauthor{Taegyun Jeon}{tgjeon@satreci.com}

\icmlkeywords{interpretability, transparency, SmoothGrad, VarGrad, Saliency Map}

\vskip 0.3in
]


\printAffiliationsAndNotice{\icmlEqualContribution{}}

\begin{abstract}
SmoothGrad and VarGrad are techniques that enhance the empirical quality of standard saliency maps by adding noise to input. However, there were few works that provide a rigorous theoretical interpretation of those methods.
We analytically formalize the result of these noise-adding methods.
As a result, we observe two interesting results from the existing noise-adding methods.
First, SmoothGrad does not make the gradient of the score function smooth.
Second, VarGrad is independent of the gradient of the score function.
We believe that our findings provide a clue to reveal the relationship between local explanation methods of deep neural networks and higher-order partial derivatives of the score function.
\end{abstract}

\section{Introduction}
\label{introduction}
Attribution methods of neural network are model interpretation methods showing how much each component of input contributes to model prediction \cite{sundararajan2017axiomatic}.
Despite the flurry of explainability research on deep neural network over the recent years, model interpretation of deep neural networks through attribution method still remains a challenging topic. Previous researches can be grouped into two categories. One approach proposes a set of propagation rules that maximize the expressiveness of the interpretation. The other approach, on the other hand, perturbs the input in a methodical (e.g.\ optimization-based masking) or a random fashion to have interpretations of better visual quality.

In light of both approaches, we emphasize the ambiguity of noise-adding methods such as SmoothGrad \cite{smilkov2017smoothgrad} and VarGrad \cite{adebayo2018local} compared to other methods. Most of the propagation-based methods can be interpreted as variants of the backpropagation algorithm \cite{ancona2017unified} and the algorithms themselves are self-explanatory. Perturbation-based methods usually optimize or manually alter the input with respect to meaningful criteria \cite{fong2017interpretable}. However, noise-adding methods merely take the mean or the variance of saliency maps generated by adding Gaussian noise to the input. Despite their apparent simplicity, the results are surprisingly effective. Ironically, the simplicity of their approach prevents us from understanding exactly how and why noise-adding methods work for the better model interpretation.

This situation poses a twofold problem. First, since the inner workings of the method are unclear, our understanding for the results produced by the noise-adding methods are also innately unclear. Second, the lack of understanding prevents others from assessing the advantages and disadvantages of such noise-adding methods.

In this paper, we address the ambiguity of noise-adding methods by applying the multivariate Taylor's theorem and some statistical theorems on SmoothGrad and VarGrad. We obtain their analytic expressions, which reveal several interesting properties. These discoveries allow us to verify intuitively plausible but opaque explanations for the effectiveness of noise-adding methods proposed in previous works. Furthermore, we formulate a general conjecture regarding reasonable model interpretations, based on our discussions.
\begin{figure*}[ht]
\vskip 0.2in
\begin{center}
\centerline{\includegraphics[width=\textwidth]{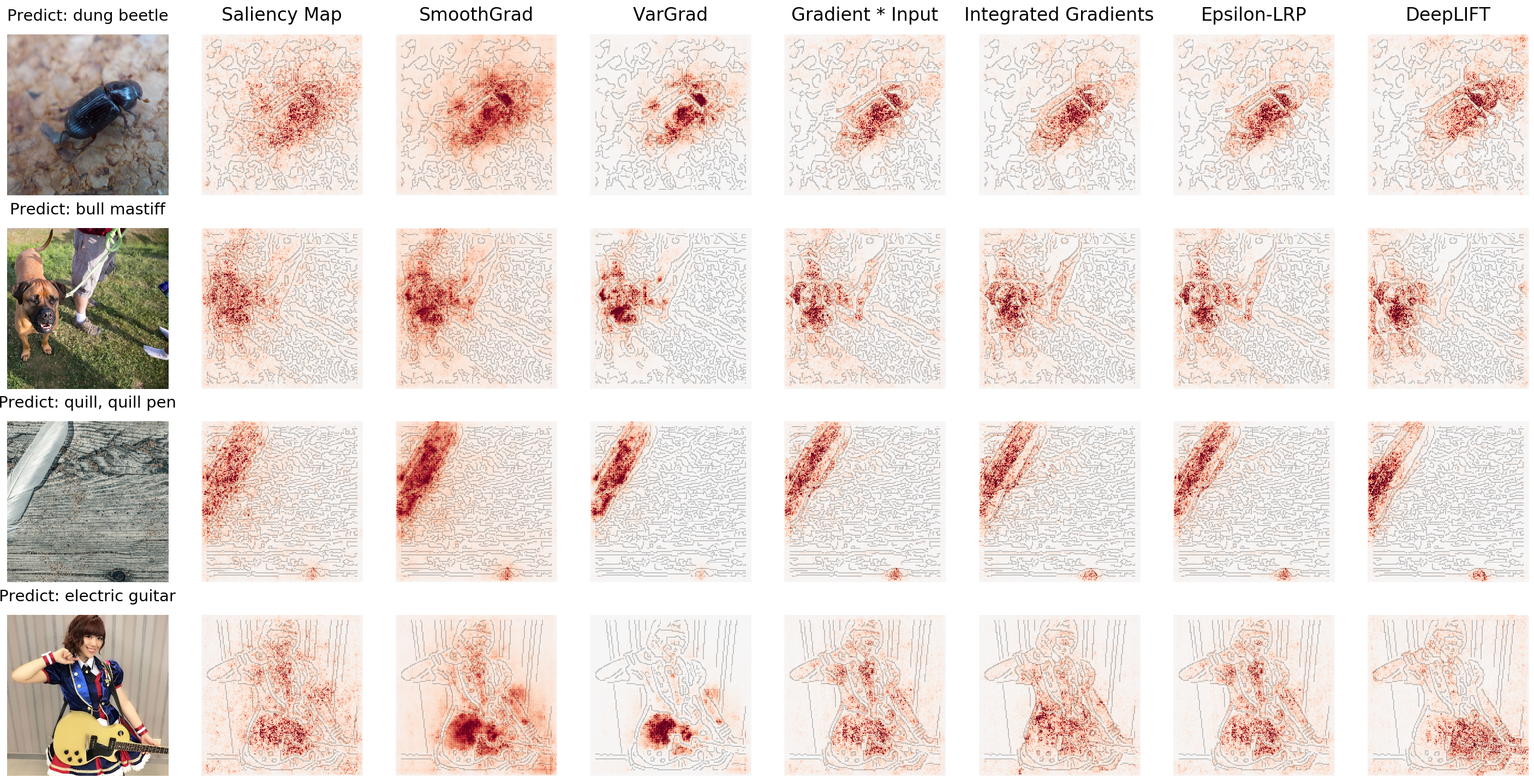}}
\caption{Sample results of saliency map, noise-adding methods, and some other gradient-based methods. For SmoothGrad and VarGrad, sampling number and standard deviation of noise are fixed to 50 and 0.025, respectively. 
The results of noise-adding methods generally look better and are less noisy than standard saliency map. This observation is consistent with discussion of the original works \cite{smilkov2017smoothgrad, adebayo2018local}. For comparison, results of the remaining four recent gradient-based methods are taken to absolute values.
See Table \ref{table} for references of four recent gradient-based attribution methods. Best seen in electric form.}
\label{figures}
\end{center}
\vskip -0.2in
\end{figure*}

Specifically, our contributions in this paper are as follows:
\begin{itemize}
\item We present non-stochastic analytic forms of approximated SmoothGrad and VarGrad and their bounds.
\item Our theorems lead to conclusions that differ from that of previous works. First, SmoothGrad does not make the gradient of the score function smooth. Second, VarGrad is independent of the gradient of the score function. In addition, their behaviors differ from that of other interpretation techniques.
\item Based on our observations, we carefully propose the conjecture that higher order partial derivatives and reasonable model interpretations are correlated.
\end{itemize}
\section{Notation}
\label{notation}
For simplicity of discussion, we limit the target network to an image classification network. Let $S:\mathbb{R}^d \mapsto \mathbb{R}^{|L|}$ be an image classification network with fixed parameters, where $x \in \mathbb{R}^d$ is a single image instance and $L$ is the set of image labels. When defining score function $S_c:\mathbb{R}^d \mapsto \mathbb{R}$ for each label candidate, $S\left(x\right) = \left[S_1\left(x\right), S_2\left(x\right),\ldots, S_{|L|}\left(x\right)\right]$. We assume that a kind of squashing function such as softmax function is applied just before the value of $S_c$ is calculated. The final classification result of $S$ for $x$ is $class_S\left(x\right) = \argmaxG_{c \in L} S_c\left(x\right)$. Note that we only consider $S_c$ to avoid using complex tensor notation.

To easily handle high-order derivatives of a multivariate function, we introduce a multi-index notation \cite{saint2017elementary}. An $d$-dimensional multi-index is a $d$-tuple $\alpha := \left(\alpha_1, \alpha_2,\ldots, \alpha_d\right)$ of non-negative integers. With this, we define
\begin{align}
|\alpha| &:=\alpha_1+\alpha_2+\ldots+\alpha_d. \\
\alpha! &:= \alpha_1!\alpha_2!\ldots \alpha_d!. \\
x^\alpha &:= x_1^{\alpha_1}x_2^{\alpha_2}\ldots x_d^{\alpha_d} \textrm{ where } x=(x_1, x_2, \ldots, x_d) \in \mathbb{R}^d.
\end{align}
If $S_c \in C^l\left(\mathbb{R}^d, \mathbb{R}\right)$ at $x_0$ and $|\alpha|  \leq  l$, a generic $|\alpha|$th-order partial derivative is denoted by
\begin{equation}
\label{eq:generic-high}
D^\alpha S_c := \frac{\partial^{|\alpha|} S_c}{\partial x_1^{\alpha_1}\ldots x_d^{\alpha_d}}.
\end{equation}
where Schwarz's theorem \cite{rudin-principles} justifies Definition \ref{eq:generic-high}.

We now define some notations for $d$-dimension $n$ noise $\epsilon$. For $k  \leq  n$, the $k$-th noise is denoted by $\epsilon_{k}=\left(\epsilon_{k1}, \epsilon_{k2}, \ldots,\epsilon_{kd}\right) \in \mathbb{R}^d$. Noise sampled independently from a probability distribution function $p(z)$ regardless of order and index is simply denoted $\epsilon \sim p(z)$. That is, $\epsilon \sim p(z)$ is equivalent to $\epsilon_{km} \sim p\left(z\right) \textrm{ } \forall k \in \left[1, n\right], \forall m \in \left[1,d\right]$.
In this cause, $\epsilon_{km}$ and $\epsilon$ are used interchangeable.

Lastly, we define simple notations for parity of integers. $\mathbb{N}^0_{even}$ means the union of the set of zero and the set of positive even numbers. $\mathbb{N}_{odd}$ refers to the set of positive odd numbers.
\section{Related Works}
\label{relate_works}
\subsection{Previous Works On Attribution Methods}
\paragraph{Saliency Map and Its Advanced Methods}
Since \cite{simonyan2013deep} first proposed using saliency maps to interpret neural networks, there have been several studies to improve propagation-based attribution method \cite{sundararajan2017axiomatic, springenberg2014striving, ancona2017unified, shrikumar2017learning, selvaraju2016grad, smilkov2017smoothgrad, bach2015pixel, montavon2017explaining, chattopadhyay2017grad}. Meanwhile, \cite{ancona2017unified} suggested the way to interpret some existing propagation-based methods as a unified gradient-based framework. On the other hand, \cite{zhang2016top, adebayo2018local} discussed the limitations of the gradient-based methodology itself.
\paragraph{Model Explanation with Perturbation}
There have been some attempts to describe the model through perturbation of input data \cite{zeiler2014visualizing, zhou2014object, cao2015look, fong2017interpretable}.
We emphasize that our major theorems and conclusions cannot be directly extended to these methods because the perturbations used in these methods are usually dependent on data or model.
\paragraph{Axiomization of Model Interpretability}
There have been several studies \cite{sundararajan2017axiomatic, ghorbani2017interpretation, adebayo2018local, kindermans2017patternnet, samek2017evaluating, dabkowski2017real} on the preferential properties or axiomization of model interpretability.
These studies are significant because they reduce ambiguity in model interpretability as research topic. 
Therefore, they are essential for an unified discussion on model interpretation.
\subsection{Brief Reviews On Our Three Topics}
\paragraph{Saliency Map} Authors of several articles \cite{erhan2009visualizing, baehrens2010explain, simonyan2013deep} proposed the saliency map, which is the partial derivative of the network output $S_c$ with respect to the input $x$, as a possible explanation for model decisions. Standard saliency map $M_c:\mathbb{R}^d \mapsto \mathbb{R}^d$ is computed by
\begin{equation}
\label{eq:saliency}
M_c\left(x\right) = \evalat[\bigg]{\frac{\partial S_c}{\partial x}}{x=x} = {\sum\limits_{|\beta|=1}}{D^\beta S_c\left(x\right)}.
\end{equation}
where $\beta$ is a multi-index.
\paragraph{SmoothGrad} The authors of \cite{smilkov2017smoothgrad} proposed SmoothGrad which calculates the average of saliency maps generated by adding Gaussian noise to the input. Compared to Equation \ref{eq:saliency}, SmoothGrad computes
\begin{equation}
\label{eq:smoothgrad}
\hat{M}_c\left(x\right) = \frac{1}{n}{\sum\limits_{k=1}^n}{M_c\left(x + \epsilon_k\right)}, \textrm{ } \epsilon \sim \mathcal{N}\left(0,\,\sigma^{2}\right).
\end{equation}
\paragraph{VarGrad} The authors of \cite{adebayo2018local} proposed VarGrad, the variance version of SmoothGrad.
\begin{align}
\label{eq:vargrad}
\widetilde{M}_c\left(x\right) &= Var\left(M_c\left(x + \epsilon_k\right)\right) \\
&= \frac{1}{n}{\sum\limits_{k=1}^n}{\left\{M_c\left(x + \epsilon_k\right)\right\}^2} - \left\{\hat{M}_c\left(x\right)\right\}^2.
\end{align}
\section{Rethinking Noise with Taylor's Theorem}
\subsection{Motivation}
Figure \ref{figures} shows the results of various attribution methods for the prediction class in inception-v3 \cite{szegedy2016rethinking} trained on ILSVRC 2013 \cite{russakovsky2015imagenet}.\footnote{The results are produced by the modified implementation of \cite{ancona2017unified}, of which the repository link is as follows: \href{https://github.com/marcoancona/DeepExplain}{https://github.com/marcoancona/DeepExplain}} As the results are shown, SmoothGrad and VarGrad seem to provide the better visual description than saliency maps in general.
More precisely, standard saliency maps overly emphasize local image regions while results produced by methods adding noise do not.
Furthermore, SmoothGrad and VarGrad produce results comparable to that of other recent gradient-based attribution methods.
The results of VarGrad are particularly sparse.
Previous studies \cite{smilkov2017smoothgrad, adebayo2018local, kindermans2017reliability} have also observed similar results.

Here our central question arises: how are the results of SmoothGrad and VarGrad different from those of the standard saliency map? For SmoothGrad, this question was covered briefly in \cite{smilkov2017smoothgrad}. The authors argued that SmoothGrad reduces the effect of `strongly fluctuating partial derivatives' on the saliency map. However, they did not offer any analytic explanation for the beneficial effect of noise on the results. As for VarGrad, its behavior is mysterious as its effectiveness. The relationship between the variability of saliency maps produced from noisy input and the saliency map produced from data is highly unclear. However, this problem has not been addressed before.

Accordingly, we attempt to answer the following questions in mathematical analysis:
\begin{itemize}
\item What is the relationship between the saliency map and the result of noise-adding methods?
\item What is the exact relationship between the result of noise-adding methods and the choice of $\sigma$?
\item Are the result of the noise-adding methods related to other factors other than the saliency map?
\end{itemize}
We express Equation \ref{eq:smoothgrad} and Equation \ref{eq:vargrad} in terms of noise parameter $\sigma$ instead of data noise $\epsilon$. 
If we cannot obtain the closed form expression of the terms for $\sigma$, we instead provide their bound as a expression for noise parameter $\sigma$.
We use the multivariate Taylor's theorem and several statistical theorems for this. Because the entire proofs are too verbose, we only write the results in this paper. The full proofs are given in the Appendix.
\subsection{SmoothGrad Does Not Make Gradient Smooth}
\label{sec:gaussian_noise}
\begin{theorem}
\label{theorem:smoothgrad}
Suppose $S_c \in C^{l+2}\left(\mathbb{R}^d, \mathbb{R}\right)$ on a closed ball $\mathbb{B}$.\footnote{Most modern neural networks are only piecewise continuously differentiable. Nonetheless, several theoretical studies \cite{funahashi1989approximate, telgarsky2016benefits, liang2016deep} have guaranteed that a general neural network can be an appropriate ($\epsilon$-)approximation of any smooth functions. Approaches for noise on neural network via Taylor's theorem are also shown in \cite{an1996effects, rifai2011adding} in the context of model regularization analysis.} If $l \geq 2$, $x \in \mathbb{B}$, $x + \epsilon_k \in \mathbb{B} \textrm{ } \forall k$, and $n$ is large enough, the result of SmoothGrad is approximated by
\begin{align}
\label{eq:appx_smoothgrad}
\hat{M}_c\left(x\right) &\approx M_c\left(x\right) \notag
\\&+ \sum\limits_{\substack{|s|=2p\\1 \leq p \leq \floor{l/2}\\\forall s_m \in \mathbb{N}^0_{even}}} \frac{D^s M_c\left(x\right)}{s!} \prod\limits_{m=1}^d \frac{s_m!}{2^{s_m/2} \left(s_m/2\right)!} \sigma^{s_m} \notag \\
&+ R_{},
\end{align}
where
\begin{align}
|R| &= \left| \frac{1}{n}{\sum\limits_{k=1}^n}\sum\limits_{|s|=l+1} \frac{D^s M_c\left(x+\eta\epsilon_k\right)}{s!} \epsilon_k^s \right|\\
&\lessapprox \sum\limits_{|s|=l+1} \frac{C}{s!}\prod\limits_{m=1}^d \frac{2^{{s_m}/2} \sigma^{s_m}}{\sqrt{\pi}} \Gamma\left(\frac{s_m}{2}+\frac{1}{2}\right).
\end{align}
with $C = \max_{\left|\alpha\right|=l+1, y \in \mathbb{B}}\left|D^\alpha M_c\left(y\right)\right|$ and some $\eta \in \left(0, 1\right)$.
\end{theorem}
\begin{proof}
See Appendix C.
\end{proof}
\subsection{VarGrad Is Independent Of Gradient}
\begin{theorem}
\label{theorem:vargrad}
Suppose $S_c \in C^{l+2}\left(\mathbb{R}^d, \mathbb{R}\right)$ on a closed ball $\mathbb{B}$. If $l$ is even, $x \in \mathbb{B}$, $x + \epsilon_k \in \mathbb{B} \textrm{ } \forall k$, and $n$ is large enough, the result of VarGrad is approximated by
\begin{align}
\label{eq:appx_vargrad}
\widetilde{M}_c(x) &\approx \sum\limits_{\substack{|s|=2p-1\\1 \leq p \leq l/2}} \left\{\frac{D^s M_c\left(x\right)}{s!}\right\}^2 \Upsilon\left(s, \sigma\right) \notag \\
&+ \sum\limits_{\substack{|s|=2p\\1 \leq p \leq l/2 \\ \forall s_m \in \mathbb{N}^0_{even}}} \left\{\frac{D^s M_c(x)}{s!}\right\}^2 \left[\Upsilon(s, \sigma)  - \left\{\Upsilon(\frac{s}{2}, \sigma)\right\}^2\right] \notag \\
&+ \sum\limits_{\substack{|s|=2p\\1 \leq p \leq l/2 \\ \exists s_m \mathbb{N}_{odd} }} \left\{\frac{D^s M_c\left(x\right)}{s!}\right\}^2 \Upsilon\left(s, \sigma\right) \notag \\
&+ R_1 + 2R_2 + 2R_3,
\end{align}
where
\begin{align}
\Upsilon\left(s, \sigma\right) = \prod\limits_{m=1}^d \frac{\left(2s_m\right)!}{2^{s_m} {s_m}!} \sigma^{2{s_m}}.
\end{align}
and $R_1$, $R_2$, $R_3$ is bounded to expression for $\sigma$. The exact equation and bound of $R_1, R_2, R_3$ is shown in Appendix D.
\end{theorem}
\begin{proof}
See Appendix D.
\end{proof}
\begin{table*}[ht]
\caption{Characteristics of formulation of noise-adding methods and some other methods. 
The column "Gradient?" indicates whether the formulation contains term for gradient.
The column "Mul w/ Input?" indicates whether the formulation contains term with input multiplied by derivate.
The column "High-order?" indicates whether the formulation contains term for high-order derivative.
See \cite{ancona2017unified} for proofs of gradient-based formulation of $\epsilon$-LRP and DeepLIFT.}
\label{sample-table}
\vskip 0.15in
\begin{center}
\begin{small}
\begin{sc}
\begin{tabular}{lcccr}
\toprule
Method & Mul w/ Input? & Gradient? & Higher-order? \\
\midrule
Saliency Map & $\times$ & $\surd$ & $\times$ \\ \hline
Gradient*Input \cite{shrikumar2016not} & $\surd$ & $\surd$ & $\times$\\
Integrated Gradient\cite{sundararajan2017axiomatic} & $\surd$ & $\surd$ & $\times$ \\
$\epsilon$-LRP \cite{bach2015pixel} & $\surd$ & $\surd$ & $\times$\\
DeepLIFT \cite{shrikumar2017learning} & $\surd$ & $\surd$ & $\times$ \\ \hline
SmoothGrad & $\times$ & $\surd$ & $\surd$\\
VarGrad & $\times$& $\times$ & $\surd$ \\
\bottomrule
\end{tabular}
\end{sc}
\end{small}
\end{center}
\label{table}
\vskip -0.1in
\end{table*}
\section{Discussions}
\label{discussions}
\subsection{Observation on Our Theorems}
\label{subsection:theorem}
\paragraph{SmoothGrad}
As pointed out in \cite{smilkov2017smoothgrad}, one of the reasons for the failure of the standard saliency map is that the partial gradient of score function for the input will act more strongly on local pixels than on global information.
Authors of \cite{smilkov2017smoothgrad} also observed that the saliency map fluctuates strongly even to small noise imperceivable to humans. Inspired by this observation, they stated SmoothGrad's motivation as follows: \textquote{Instead of basing a visualization directly on the gradient $\partial S_c$, we could base it on a smoothing of $\partial S_c$ with a Gaussian kernel.} Therefore, they argued that SmoothGrad's result looks better because SmoothGrad literally makes the gradient smooth.

Contrary to these previous discussions, our observations lead to a different conclusion.
If the discussion in \cite{smilkov2017smoothgrad} is compatible with our observation, the result of SmoothGrad when $n$ is large enough should contain a term corresponding to the smoothing effect on the saliency map $M_c(x)$.
According to Theorem \ref{theorem:smoothgrad}, that is not the case; Equation \ref{eq:smoothgrad} does not contain such a term.
Therefore, SmoothGrad does not make gradient of score function smooth from our view.
Instead, SmoothGrad is approximately the sum of the standard saliency map and the series consisting of higher-order partial derivatives and the standard deviation of the Gaussian noise.
\paragraph{VarGrad}
Although the principle of VarGrad has rarely been discussed even in the original paper \cite{adebayo2018local}, our finding about VarGrad is also counterintuitive. We can see from Theorem \ref{theorem:vargrad} that VarGrad is independent of the gradient of the score function. The result of VarGrad can be approximated as a series consisting only of higher-order partial derivatives and the standard deviation of the Gaussian noise. In other words, the result of VarGrad is not related to the saliency map.

\subsection{Comparison with Previous Discussions}
Table \ref{table} summarizes the characteristics of noise-adding methods and some other gradient-based attribution methods. 
In the table, we only deal with four recent gradient-based attribution methods listed in \cite{ancona2017unified}.
Some other gradient-based attribution methods (i.e.\ \cite{selvaraju2016grad, springenberg2014striving}) can be grouped into the same category, depending on the definition or rules of attribution.
We want to focus on the unique natures of the noise-adding methods listed in Table \ref{table}.

\paragraph{Multiplication with Input}
The presence of a term in which the input and the derivative are multiplied together has been generally taken as an important factor in sharper attribution \cite{shrikumar2017learning, sundararajan2017axiomatic, smilkov2017smoothgrad, ancona2017unified}.
Furthermore, \cite{ancona2017unified} claimed that the presence of that term makes the method a desirable global attribution method.

However, even though noise-adding methods such as VarGrad do not have these terms, their results are comparable to that of other recent attribution methods as demonstrated in Figure \ref{figures}.
Furthermore, it has been found that this term causes undesirable side effects in the attribution \cite{smilkov2017smoothgrad}, and its effect on deep neural networks (as opposed to simple linear models) is still unclear \cite{ancona2017unified}.
We therefore argue that an analytic approach to the need for multiplication with input is necessary.
\paragraph{Gradient}
On the presence or the absence of the gradient term, our findings are even more surprising. Since \cite{simonyan2013deep} first introduced model interpretation by saliency maps, all following propagation-based attribution methods have used the gradient in some way. However, our findings suggest that SmoothGrad and VarGrad deviate from this trend, as mentioned in Section \ref{subsection:theorem}.
\paragraph{Higher-order Derivative}
Taken together, our theorems suggest that a major factor affecting the result of noise-adding methods is the higher-order partial derivatives of the score function for the data point, not just the saliency map.
Despite conflicts between our conclusions and that of other works, it is undeniable that the noise-adding methods are qualitatively better than the standard saliency map.
To account for this phenomenon, we cautiously propose the conjecture that there may be a correlation between higher order partial derivatives of model function and the attributions defined from sensible axioms of model interpretability.

There is few articles that focus on the higher-order partial derivative of the model function for model explanation.
One notable exception is \cite{koh2017understanding}, which studied the influence function via Jacobian-Hessian products of the model.
The purpose of \cite{koh2017understanding}, however, is not to take the model attribution of the input but to find the responsible training data through the influence function.
As far as we know, \cite{chattopadhyay2017grad} is the study of model attribution method that is most related to higher-order derivatives.
In \cite{chattopadhyay2017grad}, computation of higher derivatives is required for getting the gradient weights in more principled way than class activation map \cite{zhou2016learning} or Grad-CAM \cite{selvaraju2016grad}.
We hope for further advanced discussions on our view in the future, with a legitimate axiom on model interpretation.
\subsection{Inaccessibility to Experimentation}
It is worth mentioning that direct computation of Equation \ref{eq:appx_smoothgrad} and Equation \ref{eq:appx_vargrad} is numerically intractable.
There are two reasons for this claim. First, it requires the calculation of an $d$-dimension explicitly restricted partition set $s$ \cite{stanley1986enumerative} with increasing order.
Additionally, $|s|$-order partial derivative of the score function $M_c$ should be computed for all possible multi-indexes.
Both are practically difficult to compute.
Despite the inaccessibility to experimentation, our view over noise-adding methods allows theoretically interesting discussions.
\section{Conclusions}
\label{conclusions}
We explored the analytic form of SmoothGrad and VarGrad, variants of the saliency map. Our conclusions about the behavior of both methods when the sample number is sufficient were conflicted with the existing view. First, SmoothGrad does not make gradient of score function smooth. Second, VarGrad is independent of gradient of score function.

To reconcile the success of noise-adding methods and our conclusions, we carefully presented a conjecture: there may be a correlation between higher order partial derivatives of the model function and a sensible model interpretation. We hope to see advanced discussions on model interpretation from this perspective in the future.

\bibliography{content_bio}
\bibliographystyle{icml2018}

\renewcommand\thesection{\Alph{section}}
\renewcommand\thesubsection{\thesection.\arabic{subsection}}

\onecolumn
\icmltitle{Noise-adding Methods of Saliency Map\\as Series of Higher Order Partial Derivative\\: Appendix}
\pagenumbering{gobble}
\icmlsetsymbol{equal}{*}

\begin{icmlauthorlist}
\icmlauthor{Junghoon Seo}{}
\icmlauthor{Jeongyeol Choe}{}
\icmlauthor{Jamyoung Koo}{}
\icmlauthor{Seunghyeon Jeon}{}
\icmlauthor{Beomsu Kim}{}
\icmlauthor{Taegyun Jeon}{}
\end{icmlauthorlist}





\setcounter{theorem}{0}
\setcounter{table}{0}
\setcounter{equation}{0}

\vskip 0.3in
\setcounter{section}{0}
\section{Notation Table}
We summarize the basic notations for symbols not introduced in Section 2 of the main paper in the following table.
For non-shown notations, see Section 2 of the main paper.
\begin{table}[h]
\centering
\caption{Notation of Basic Symbols}
\label{my-label}
\begin{tabular}{c|l}
$\mathbb{E}$         & Sample mean, or expectation of a random variable \\
$Var(\cdot)$         & Sample variance, or population variance \\
$Cov(\cdot)$         & Sample covariance, or population covariance \\
$\approx$            & Approximately equal to, especially used in context that sample number is large enough \\
$\lessapprox$        & Approximately less than or equal to, especially used in context that sample number is large enough \\
$\mathbb{N}$         & Normal distribution \\
$\Gamma(.)$             & Gamma function \\
$\sim$               & "has the probability distribution of" \\
$C^k(\cdot, \cdot)$  & Function class of which the first $k$ derivatives all exist and are continuous \\
$\ceil{\cdot}$       & Ceil function \\
$\floor{\cdot}$      & Floor function \\
$\left|\cdot\right|$ & Absolute function \\
$Z$                  & Distribution of population
\end{tabular}
\end{table}
\section{Lemmata}
Before proving the main theorems, we propose some lemmata that should be seen proactively.
\begin{lemma}
\label{lemma:odd}
Suppose $Z_m \sim p(z) \textrm{ for } 1 \leq m \leq d$, where $p\left(z\right)$ is the symmetric probability distribution with zero mean. Define $\left|s\right|=\sum_{m=1}^d s_m$ for each $s_m$ is non-negative integer. If $\left|s\right|$ is odd,
\begin{equation}
\mathbb{E}\left[Z_1^{s_1}Z_2^{s_2}\ldots Z_d^{s_d}\right] = 0.
\end{equation}
In addition, to be $\mathbb{E}\left[Z_1^{s_1}Z_2^{s_2}\ldots Z_d^{s_d}\right] \neq 0$, all $s_d$ must be zero or even.
\end{lemma}
\begin{proof}
Because $\left|s\right|$ is odd, at least one of $s_m$ is odd.
Since each $Z_m$ is sampled independently, $\mathbb{E}\left[Z_1^{s_1}Z_2^{s_2}\ldots Z_d^{s_d}\right]=\mathbb{E}\left[Z_1^{s_1}\right]\mathbb{E}\left[Z_2^{s_2}\right]\ldots\mathbb{E}\left[Z_d^{s_d}\right]$. For any $m$ and odd number $\lambda$, $\mathbb{E}\left[Z_m^\lambda\right]=0$ because $p\left(z\right)$ is the symmetric probability distribution with zero mean. 
Thus, at least one $\mathbb{E}\left[Z_m^{s_m}\right]$ is zero.
\end{proof}
\begin{corollary}
\label{corollary}
Suppose $\epsilon \sim p\left(z\right)$, where $p\left(z\right)$ is the symmetric probability distribution with zero mean. Define $\left|s\right|=\sum_{m=1}^d s_m$ for each $s_m$ is non-negative integer. If $\left|s\right|$ is odd,
\begin{equation}
\frac{\sum_{k=1}^n \prod_{m=1}^d{\epsilon_{km}^{s_m}}}{n} \approx 0,
\end{equation}
where $n$ is large enough. 
In addition, to be $\sum_{k=1}^n \prod_{m=1}^d{\epsilon_{km}^{s_m}}/n \not\approx 0$ when n is large, all $s_m$ must be zero or even.
\end{corollary}
\begin{proof}
It is clear from Lemma \ref{lemma:odd} and its proof.
\end{proof}
\begin{lemma}
\label{lemma:expectation}
Suppose $s$ is a non-negative integer. If $\epsilon \sim \mathcal{N}\left(0,\,\sigma^{2}\right)$,
\begin{align}
\mathbb{E}\left[\epsilon^{2s}\right] &= \frac{\left(2s\right)!}{2^s s!} \sigma^{2s}. \\
\mathbb{E}\left[\left|\epsilon^{s}\right|\right] &= \frac{2^{s/2} \sigma^{s}}{\sqrt{\pi}} \Gamma\left(\frac{s}{2}+\frac{1}{2}\right).
\end{align}
\end{lemma}
\begin{proof}
Suppose $s \geq 1$. By the Law Of The Unconscious Statistician,
\begin{align}
\mathbb{E}\left[\epsilon^{2s}\right] &= \int_{-\infty}^{+\infty} x^{2s} \frac{1}{\sqrt{2\pi\sigma^2}} \exp^{-\frac{x^2}{2\sigma^2}} dx \\
&= \frac{2^s \sigma^{2s}}{\sqrt{\pi}} \Gamma\left(s+\frac{1}{2}\right) = \frac{\left(2s\right)!}{2^s s!} \sigma^{2s}.
\end{align}
\begin{align}
\mathbb{E}\left[\left|\epsilon^{s}\right|\right] &= \int_{-\infty}^{+\infty} \left|x^s\right| \frac{1}{\sqrt{2\pi\sigma^2}} \exp^{-\frac{x^2}{2\sigma^2}} dx \\
&= \frac{2^{s/2} \sigma^{s}}{\sqrt{\pi}} \Gamma\left(\frac{s}{2}+\frac{1}{2}\right).
\end{align}
These formulas are also true for $s=0$:
\begin{align}
\mathbb{E}\left[\epsilon^{0}\right] = \mathbb{E}\left[\left|\epsilon^{0}\right|\right] = \mathbb{E}\left[1\right] = 1 = \frac{0!}{2^0 0!} \sigma^{0} = \frac{2^{0/2} \sigma^{0}}{\sqrt{\pi}} \Gamma\left(\frac{0}{2}+\frac{1}{2}\right).
\end{align}
\end{proof}
\begin{lemma}
\label{lemma:bounded_var}
Suppose two arbitrary random variables $X, Y$, and $k$-th sample $x_k$ from $X$. If $\mathbb{E}\left[Y\right] = 0$ and $\left|x_k\right| < C\textrm{ }\forall k$,
\begin{equation}
Var\left(XY\right) \leq Var\left(CY\right) = C^2Var\left(Y\right).
\end{equation}
\begin{proof}
\begin{align}
Var\left(XY\right) &=\mathbb{E}\left[X^2Y^2\right]-\{\mathbb{E}\left[XY\right]\}^2. \\
Var\left(CY\right) &=C^2\mathbb{E}\left[Y^2\right]-C^2\{\mathbb{E}\left[Y\right]\}^2 = C^2\mathbb{E}\left[Y^2\right].
\end{align}
\begin{align}
Var\left(CY\right) &- Var\left(XY\right) = \left(\mathbb{E}\left[C^2Y^2\right] -\mathbb{E}\left[X^2Y^2\right]\right) + \left\{\mathbb{E}\left[XY\right]\right\}^2 \geq 0.
\end{align}
\end{proof}
\end{lemma}
\begin{lemma}
\label{lemma:covarvar}
Suppose two arbitrary random variables $X, Y$. Then,
\begin{equation}
\left| Cov\left(X, Y\right) \right| \leq \sqrt{Var\left(X\right) Var\left(Y\right)}.
\end{equation}
\end{lemma}
\begin{proof}
This can be proved via Cauchy-Schwarz inequality. Refer \citeNew{fujii1997operator} for detail.
\end{proof}

\section{Proof of Theorem 1}
\label{appendix:theorem1}
\begin{proof}
According to conditions, $M_c \in C^{l+1}(\mathbb{R}^d, \mathbb{R})$. Starting from multivariate Taylor's theorem \citeNew{trench2013introduction} of Equation of SmoothGrad,
\begin{align}
\hat{M}_c\left(x\right) &= \frac{1}{n}{\sum\limits_{k=1}^n} {M_c\left(x\right)+ \sum\limits_{1 \leq |s| \leq l} \frac{D^s M_c\left(x\right)}{s!} \epsilon_k^s + \sum\limits_{|s|=l+1} \frac{D^s M_c\left(x+\eta\epsilon_k\right)}{s!} \epsilon_k^s} \\
&= M_c\left(x\right)+ \frac{1}{n}{\sum\limits_{k=1}^n}\sum\limits_{1 \leq |s| \leq l} \frac{D^s M_c\left(x\right)}{s!} \epsilon_k^s + \frac{1}{n}{\sum\limits_{k=1}^n}\sum\limits_{|s|=l+1} \frac{D^s M_c\left(x+\eta\epsilon_k\right)}{s!} \epsilon_k^s \\
&= M_c\left(x\right)+ \frac{1}{n}{\sum\limits_{k=1}^n}\sum\limits_{1 \leq |s| \leq l} \frac{D^s M_c\left(x\right)}{s!} \prod\limits_{m=1}^d{\epsilon_{km}^{s_m}} + \frac{1}{n}{\sum\limits_{k=1}^n}\sum\limits_{|s|=l+1} \frac{D^s M_c\left(x+\eta\epsilon_k\right)}{s!} \prod\limits_{m=1}^d{\epsilon_{km}^{s_m}} \\
\label{eq:taylor} &= M_c\left(x\right)+ \sum\limits_{1 \leq |s| \leq l} \frac{D^s M_c\left(x\right)}{s!} \frac{{\sum_{k=1}^n}\prod_{m=1}^d{\epsilon_{km}^{s_m}}}{n} + \sum\limits_{|s|=l+1} \frac{1}{s!} \frac{{\sum_{k=1}^n} D^s M_c\left(x+\eta\epsilon_k\right) \prod_{m=1}^d{\epsilon_{km}^{s_m}}}{n},
\end{align}
for some $\eta \in \left(0, 1\right)$.
The second term of Equation \ref{eq:taylor} can be rearranged as
\begin{align}
\label{eq:odd_second}
\sum\limits_{\substack{|s|=2p-1\\1 \leq p \leq \ceil{l/2}}} \frac{D^s M_c\left(x\right)}{s!} \frac{{\sum_{k=1}^n}\prod_{m=1}^d{\epsilon_{km}^{s_m}}}{n} + \sum\limits_{\substack{|s|=2p\\1 \leq p \leq \floor{l/2}}} \frac{D^s M_c\left(x\right)}{s!} \frac{{\sum_{k=1}^n}\prod_{m=1}^d{\epsilon_{km}^{s_m}}}{n}.
\end{align}
Meanwhile, due to the continuity of $l+1$-th order partial derivatives in the compact set $\mathbb{B}$, we can obtains the uniform bound of the third term of Equation \ref{eq:taylor} as follows:
\begin{equation}
\label{eq:third_le}\left| \sum\limits_{|s|=l+1} {\frac{1}{s!}} \frac{{\sum_{k=1}^n} D^s M_c\left(x+\eta\epsilon_k\right) \prod_{m=1}^d{ \epsilon_{km}^{s_m}}}{n}\right| \leq \sum\limits_{|s|=l+1} \frac{C}{s!}\frac{{\sum_{k=1}^n}\prod_{m=1}^d{\left| \epsilon_{km}^{s_m} \right| }}{n},
\end{equation}
where $C = \max_{|\alpha|=l+1, y \in \mathbb{B}}\left|D^\alpha M_c\left(y\right)\right|$. Note that $\left| D^s M_c\left(x+\eta\epsilon_k\right) \right| \leq C \textrm{ } \forall k$ when $\left| s \right| = l+1$.

Next, we arrange the terms of Equation \ref{eq:odd_second} and Equation \ref{eq:third_le} in order.
Recall that all $d$ elements of $\epsilon_k$ are sampled independently and identically.
Let $Z_m$ be $Z_m \sim \mathcal{N}\left(0,\,\sigma^{2}\right)$ for $1 \leq m \leq d$. When $n$ is large enough,
\begin{align}
\sum\limits_{\substack{|s|=2p-1\\1 \leq p \leq \ceil{l/2}}} \frac{D^s M_c\left(x\right)}{s!} \frac{{\sum_{k=1}^n}\prod_{m=1}^d{\epsilon_{km}^{s_m}}}{n} &\approx \sum\limits_{\substack{|s|=2p-1\\1 \leq p \leq \ceil{l/2}}} \frac{D^s M_c\left(x\right)}{s!} \mathbb{E}\left[\prod\limits_{m=1}^d Z_m^{s_m}\right]
\\
&= \sum\limits_{\substack{|s|=2p-1\\1 \leq p \leq \ceil{l/2}}} \frac{D^s M_c\left(x\right)}{s!} \prod\limits_{m=1}^d \mathbb{E}\left[Z_m^{s_m}\right] = 0. \\
\sum\limits_{\substack{|s|=2p\\1 \leq p \leq \floor{l/2}}} \frac{D^s M_c\left(x\right)}{s!} \frac{{\sum_{k=1}^n}\prod_{m=1}^d{\epsilon_{km}^{s_m}}}{n} &\approx \sum\limits_{\substack{|s|=2p\\1 \leq p \leq \floor{l/2}\\\forall s_m\textrm{ is zero or even}}} \frac{D^s M_c\left(x\right)}{s!} \mathbb{E}\left[\prod\limits_{m=1}^d Z_m^{s_m}\right]
\\
&= \sum\limits_{\substack{|s|=2p\\1 \leq p \leq \floor{l/2}\\\forall s_m\textrm{ is zero or even}}} \frac{D^s M_c\left(x\right)}{s!} \prod\limits_{m=1}^d \mathbb{E}\left[Z_m^{s_m}\right] \\
&= \sum\limits_{\substack{|s|=2p\\1 \leq p \leq \floor{l/2}\\\forall s_m\textrm{ is zero or even}}} \frac{D^s M_c\left(x\right)}{s!} \prod\limits_{m=1}^d \frac{s_m!}{2^{s_m/2} \left(s_m/2\right)!} \sigma^{s_m}. \\
\sum\limits_{|s|=l+1} \frac{C}{s!}\frac{{\sum_{k=1}^n}\prod_{m=1}^d{\left| \epsilon_{km}^{s_m} \right| }}{n} &\approx 
\sum\limits_{|s|=l+1} \frac{C}{s!}\mathbb{E}\left[\prod\limits_{m=1}^d \left| Z_m^{s_m} \right| \right] \\
&= \sum\limits_{|s|=l+1} \frac{C}{s!}\prod\limits_{m=1}^d \mathbb{E}\left[ \left| Z_m^{s_m} \right| \right] \\
&= \sum\limits_{|s|=l+1} \frac{C}{s!}\prod\limits_{m=1}^d \frac{2^{{s_m}/2} \sigma^{s_m}}{\sqrt{\pi}} \Gamma\left(\frac{s_m}{2}+\frac{1}{2}\right).
\end{align}
from Lemma \ref{lemma:odd}, Lemma \ref{lemma:expectation} and Corollary \ref{corollary}.

As a result,
\begin{align}
\therefore \hat{M}_c\left(x\right) \approx M_c\left(x\right) &+ \sum\limits_{\substack{|s|=2p\\1 \leq p \leq \floor{l/2}\\\forall s_m\textrm{ is zero or even}}} \frac{D^s M_c\left(x\right)}{s!} \prod\limits_{m=1}^d \frac{s_m!}{2^{s_m/2} \left(s_m/2\right)!} \sigma^{s_m} + R,
\end{align}
\begin{align}
|R| = \left| \frac{1}{n}{\sum\limits_{k=1}^n}\sum\limits_{|s|=l+1} \frac{D^s M_c\left(x+\eta\epsilon_k\right)}{s!} \epsilon_k^s \right| &\lessapprox \sum\limits_{|s|=l+1} \frac{C}{s!}\prod\limits_{m=1}^d \frac{2^{{s_m}/2} \sigma^{s_m}}{\sqrt{\pi}} \Gamma\left(\frac{s_m}{2}+\frac{1}{2}\right).
\end{align}
\end{proof}

\section{Proof of Theorem 2}
\label{appendix:threorem2}
\begin{proof}
Starting from Definition of VarGrad and multivariate Taylor's theorem,
\begin{align} \label{eq13}
\widetilde{M}_c\left(x\right) &= \frac{1}{n}{\sum\limits_{k=1}^n} \left\{M_c\left(x\right)+\left( \sum\limits_{1 \leq |s| \leq l} \frac{D^s M_c(x)}{s!} \epsilon_k^s + \sum\limits_{|s|=l+1} \frac{D^s M_c\left(x+\eta\epsilon_k\right)}{s!} \epsilon_k^s\right)\right\}^2 \notag \\
&-\left\{M_c\left(x\right)+\frac{1}{n}{\sum\limits_{k=1}^n}\left(\sum\limits_{1 \leq |s| \leq l} \frac{D^s M_c\left(x\right)}{s!} \epsilon_k^s + \sum\limits_{|s|=l+1} \frac{D^s M_c\left(x+\eta\epsilon_k\right)}{s!} \epsilon_k^s\right)\right\}^2 \\
&= \frac{1}{n}{\sum\limits_{k=1}^n}\left\{ \sum\limits_{1 \leq |s| \leq l} \frac{D^s M_c(x)}{s!} \epsilon_k^s + \sum\limits_{|s|=l+1} \frac{D^s M_c(x+\eta\epsilon_k)}{s!} \epsilon_k^s\right\}^2 \notag \\
&-\left\{\frac{1}{n}{\sum\limits_{k=1}^n}(\sum\limits_{1 \leq |s| \leq l} \frac{D^s M_c(x)}{s!} \epsilon_k^s + \sum\limits_{|s|=l+1} \frac{D^s M_c(x+\eta\epsilon_k)}{s!} \epsilon_k^s)\right\}^2 \\
&= Var\left(\sum\limits_{1 \leq |s| \leq l} \frac{D^s M_c\left(x\right)}{s!} \epsilon_k^s + \sum\limits_{|s|=l+1} \frac{D^s M_c\left(x+\eta\epsilon_k\right)}{s!} \epsilon_k^s\right) \\
\label{eq:34}&= Var\left(\sum\limits_{\substack{|s|=2p-1\\1 \leq p \leq \ceil{l/2}}} \frac{D^s M_c\left(x\right)}{s!} \epsilon_k^s + \sum\limits_{\substack{|s|=2p\\1 \leq p \leq \floor{l/2}}} \frac{D^s M_c\left(x\right)}{s!} \epsilon_k^s + \sum\limits_{|s|=l+1} \frac{D^s M_c\left(x+\eta\epsilon_k\right)}{s!} \epsilon_k^s\right).
\end{align}
By the fact that the variance of sum of random variables is the sum of their covariances, Equation \ref{eq:34} is expanded as
\begin{align}
\label{eq:expand}&\sum\limits_{\substack{|s|=2p-1\\1 \leq p \leq \ceil{l/2}}} \left\{\frac{D^s M_c\left(x\right)}{s!}\right\}^2 Var\left(\epsilon_k^s\right) + \sum\limits_{\substack{|s|=2p\\1 \leq p \leq \floor{l/2}}} \left\{\frac{D^s M_c\left(x\right)}{s!}\right\}^2 Var\left(\epsilon_k^s\right) \notag \\ + &\sum\limits_{|s|=l+1} \left\{\frac{1}{s!}\right\}^2 Var\left(D^s M_c\left(x+\eta\epsilon_k\right)\epsilon_k^s\right) \notag \\ + &2\sum\limits_{\substack{|\alpha|=2p_1-1\\1 \leq p_1 \leq \ceil{l/2}}}\sum\limits_{\substack{|\beta|=2p_2\\1 \leq p_2 \leq \floor{l/2}}} \frac{D^\alpha M_c\left(x\right)}{\alpha!} \frac{D^\beta M_c\left(x\right)}{\beta!} Cov\left(\epsilon_k^\alpha, \epsilon_k^\beta\right) \notag \\ + &2\sum\limits_{\substack{|\alpha|=2p-1\\1 \leq p \leq \ceil{l/2}}}\sum\limits_{|\beta|=l+1} \frac{D^\alpha M_c\left(x\right)}{\alpha!} \frac{1}{\beta!} Cov\left(\epsilon_k^\alpha, D^\beta M_c(x+\eta\epsilon_k) \epsilon_k^\beta\right) \notag \\
+ &2\sum\limits_{\substack{|\alpha|=2p\\1 \leq p \leq \floor{l/2}}}\sum\limits_{|\beta|=l+1} \frac{D^\alpha M_c\left(x\right)}{\alpha!} \frac{1}{\beta!} Cov\left(\epsilon_k^\alpha, D^\beta M_c(x+\eta\epsilon_k) \epsilon_k^\beta\right).
\end{align}
By arranging three residual-free terms of Equation \ref{eq:expand} in order, we get
\begin{align}
\label{eq:free1}\sum\limits_{\substack{|s|=2p-1\\1 \leq p \leq \ceil{l/2}}} \left\{\frac{D^s M_c\left(x\right)}{s!}\right\}^2 Var\left(\epsilon_k^s\right) &= \sum\limits_{\substack{|s|=2p-1\\1 \leq p \leq \ceil{l/2}}} \left\{\frac{D^s M_c\left(x\right)}{s!}\right\}^2 Var\left(\prod\limits_{m=1}^d{\epsilon_{km}^{s_m}}\right). \\
\label{eq:free2}\sum\limits_{\substack{|s|=2p\\1 \leq p \leq \floor{l/2}}} \left\{\frac{D^s M_c\left(x\right)}{s!}\right\}^2 Var\left(\epsilon_k^s\right) &= \sum\limits_{\substack{|s|=2p\\1 \leq p \leq \floor{l/2}}} \left\{\frac{D^s M_c\left(x\right)}{s!}\right\}^2 Var\left(\prod\limits_{m=1}^d{\epsilon_{km}^{s_m}}\right).
\end{align}
\begin{align}
\label{eq:free3}&\sum\limits_{\substack{|\alpha|=2p_1-1\\1 \leq p_1 \leq \ceil{l/2}}}\sum\limits_{\substack{|\beta|=2p_2\\1 \leq p_2 \leq \floor{l/2}}} \frac{D^\alpha M_c\left(x\right)}{\alpha!} \frac{D^\beta M_c\left(x\right)}{\beta!} Cov\left(\epsilon_k^\alpha, \epsilon_k^\beta\right) \notag \\ &= \sum\limits_{\substack{|\alpha|=2p_1-1\\1 \leq p_1 \leq \ceil{l/2}}}\sum\limits_{\substack{|\beta|=2p_2\\1 \leq p_2 \leq \floor{l/2}}} \frac{D^\alpha M_c\left(x\right)}{\alpha!} \frac{D^\beta M_c\left(x\right)}{\beta!} Cov\left(\prod\limits_{m=1}^d{\epsilon_{km}^{\alpha_m}}, \prod\limits_{m=1}^d{\epsilon_{km}^{\beta_m}}\right).
\end{align}
Next, we arrange the terms of Equation \ref{eq:free1}, Equation \ref{eq:free2}, and Equation \ref{eq:free3} in order.
Recall that all $d$ elements of $\epsilon_k$ are sampled independently and identically.
Let $Z_m$ be $Z_m \sim \mathcal{N}(0,\,\sigma^{2})$ for $1 \leq m \leq d$. When $n$ is large enough,
\begin{align}
&\sum\limits_{\substack{|s|=2p-1\\1 \leq p \leq \ceil{l/2}}} \left\{\frac{D^s M_c\left(x\right)}{s!}\right\}^2 Var\left(\prod\limits_{m=1}^d{\epsilon_{km}^{s_m}}\right) \approx \sum\limits_{\substack{|s|=2p-1\\1 \leq p \leq \ceil{l/2}}} \left\{\frac{D^s M_c\left(x\right)}{s!}\right\}^2 Var\left(\prod\limits_{m=1}^d{Z_m^{s_m}}\right) \\
&= \sum\limits_{\substack{|s|=2p-1\\1 \leq p \leq \ceil{l/2}}} \left\{\frac{D^s M_c\left(x\right)}{s!}\right\}^2 \left\{\mathbb{E}\left[\prod\limits_{m=1}^d{Z_m^{2s_m}}\right] - \left(\mathbb{E}\left[\prod\limits_{m=1}^d{Z_m^{s_m}}\right]\right)^2 \right\} \\
&= \sum\limits_{\substack{|s|=2p-1\\1 \leq p \leq \ceil{l/2}}} \left\{\frac{D^s M_c\left(x\right)}{s!}\right\}^2 \left\{\prod\limits_{m=1}^d \mathbb{E}\left[{Z_m^{2s_m}}\right] -  \left(\prod\limits_{m=1}^d\mathbb{E}\left[{Z_m^{s_m}}\right]\right)^2\right\} \\
&= \sum\limits_{\substack{|s|=2p-1\\1 \leq p \leq \ceil{l/2}}} \left\{\frac{D^s M_c\left(x\right)}{s!}\right\}^2 \prod\limits_{m=1}^d \frac{\left(2s_m\right)!}{2^{s_m} {s_m}!} \sigma^{2{s_m}},
\end{align}
from Lemma \ref{lemma:odd} and Lemma \ref{lemma:expectation}.

When Equation \ref{eq:free2} is treated in the same manner,
\begin{align}
&\sum\limits_{\substack{|s|=2p\\1 \leq p \leq \floor{l/2}}} \left\{\frac{D^s M_c\left(x\right)}{s!}\right\}^2 Var\left(\prod\limits_{m=1}^d{\epsilon_{km}^{s_m}}\right) \\
&\approx \sum\limits_{\substack{|s|=2p\\1 \leq p \leq \floor{l/2}}} \left\{\frac{D^s M_c\left(x\right)}{s!}\right\}^2 \left\{\prod\limits_{m=1}^d \mathbb{E}\left[{Z_m^{2 s_m}}\right] - \left(\prod\limits_{m=1}^d\mathbb{E}\left[{Z_m^{s_m}}\right]\right)^2\right\} \\
&= \sum\limits_{\substack{|s|=2p\\1 \leq p \leq \floor{l/2} \\ \forall s_m\textrm{ is zero or even}}} \left\{\frac{D^s M_c\left(x\right)}{s!}\right\}^2 \left\{\prod\limits_{m=1}^d \mathbb{E}\left[{Z_m^{2 s_m}}\right] - \left(\prod\limits_{m=1}^d\mathbb{E}\left[{Z_m^{s_m}}\right]\right)^2\right\} \notag \\
&+ \sum\limits_{\substack{|s|=2p\\1 \leq p \leq \floor{l/2} \\ \exists s_m\textrm{ is odd}}} \left\{\frac{D^s M_c(x)}{s!}\right\}^2 \left\{\prod\limits_{m=1}^d \mathbb{E}\left[{Z_m^{2 s_m}}\right] - \left(\prod\limits_{m=1}^d\mathbb{E}\left[{Z_m^{s_m}}\right]\right)^2\right\} \\
&= \sum\limits_{\substack{|s|=2p\\1 \leq p \leq \floor{l/2} \\ \forall s_m\textrm{ is zero or even}}} \left\{\frac{D^s M_c\left(x\right)}{s!}\right\}^2 \left\{\prod\limits_{m=1}^d \frac{\left(2s_m\right)!}{2^{s_m} {s_m}!} \sigma^{2{s_m}}  - \prod\limits_{m=1}^d \frac{\{\left(s_m\right)!\}^2}{2^{s_m} \{\left(s_m/2\right)!\}^2} \sigma^{{2s_m}}\right\} \notag \\
&+ \sum\limits_{\substack{|s|=2p\\1 \leq p \leq \floor{l/2} \\ \exists s_m\textrm{ is odd}}} \left\{\frac{D^s M_c(x)}{s!}\right\}^2 \prod\limits_{m=1}^d \frac{\left(2s_m\right)!}{2^{s_m} {s_m}!} \sigma^{2{s_m}}.
\end{align}

Finally, for Equation \ref{eq:free3},
\begin{align}
&\sum\limits_{\substack{|\alpha|=2p_1-1\\1 \leq p_1 \leq \ceil{l/2}}}\sum\limits_{\substack{|\beta|=2p_2\\1 \leq p_2 \leq \floor{l/2}}} \frac{D^\alpha M_c\left(x\right)}{\alpha!} \frac{D^\beta M_c\left(x\right)}{\beta!} Cov\left(\prod\limits_{m=1}^d{\epsilon_{km}^{\alpha_m}}, \prod\limits_{m=1}^d{\epsilon_{km}^{\beta_m}}\right) \\
&\approx \sum\limits_{\substack{|\alpha|=2p_1-1\\1 \leq p_1 \leq \ceil{l/2}}}\sum\limits_{\substack{|\beta|=2p_2\\1 \leq p_2 \leq \floor{l/2}}} \frac{D^\alpha M_c\left(x\right)}{\alpha!} \frac{D^\beta M_c\left(x\right)}{\beta!} Cov\left(\prod\limits_{m=1}^d{Z_m^{\alpha_m}}, \prod\limits_{m=1}^d{Z_m^{\beta_m}}\right) \\
&= \sum\limits_{\substack{|\alpha|=2p_1-1\\1 \leq p_1 \leq \ceil{l/2}}}\sum\limits_{\substack{|\beta|=2p_2\\1 \leq p_2 \leq \floor{l/2}}} \frac{D^\alpha M_c\left(x\right)}{\alpha!} \frac{D^\beta M_c\left(x\right)}{\beta!} \left\{\mathbb{E}\left[\prod\limits_{m=1}^d{Z_m^{\alpha_m + \beta_m}}\right] - \mathbb{E}\left[\prod\limits_{m=1}^d Z_m^{\alpha_m}\right]\mathbb{E}\left[\prod\limits_{m=1}^d Z_m^{\beta_m}\right]\right\} \\
& = \sum\limits_{\substack{|\alpha|=2p_1-1\\1 \leq p_1 \leq \ceil{l/2}}}\sum\limits_{\substack{|\beta|=2p_2\\1 \leq p_2 \leq \floor{l/2}}} \frac{D^\alpha M_c\left(x\right)}{\alpha!} \frac{D^\beta M_c\left(x\right)}{\beta!} \left\{\prod\limits_{m=1}^d\mathbb{E}\left[{Z_m^{\alpha_m + \beta_m}}\right] - \prod\limits_{m=1}^d \mathbb{E}\left[Z_m^{\alpha_m}\right] \prod\limits_{m=1}^d \mathbb{E}\left[Z_m^{\beta_m}\right]\right\} \\
&= \sum\limits_{\substack{|\alpha|=2p_1-1\\1 \leq p_1 \leq \ceil{l/2}}}\sum\limits_{\substack{|\beta|=2p_2\\1 \leq p_2 \leq \floor{l/2} \\ \forall \beta_m: \alpha_m + \beta_m\textrm{ is zero or even}}} \frac{D^\alpha M_c\left(x\right)}{\alpha!} \frac{D^\beta M_c\left(x\right)}{\beta!} \prod\limits_{m=1}^d\mathbb{E}\left[{Z_m^{\alpha_m + \beta_m}}\right] \\
&= \sum\limits_{\substack{|\alpha|=2p_1-1\\1 \leq p_1 \leq \ceil{l/2}}}\sum\limits_{\substack{|\beta|=2p_2\\1 \leq p_2 \leq \floor{l/2} \\ \forall \beta_m\textrm{ has same parity with }\alpha_m}} \frac{D^\alpha M_c\left(x\right)}{\alpha!} \frac{D^\beta M_c\left(x\right)}{\beta!} \prod\limits_{m=1}^d\mathbb{E}\left[{Z_m^{\alpha_m + \beta_m}}\right] \\
&= 0.
\end{align}

Then, we try to get the bounds of the remaining three terms of Equation \ref{eq:expand}. 
First of all, we arrange the third term of Equation \ref{eq:expand}.
Set $C = \max_{\left|\alpha\right|=l+1, y \in \mathbb{B}}\left|D^\alpha M_c\left(y\right)\right|$. Because $\left| D^s M_c\left(x+\eta\epsilon_k\right) \right| \leq C$ and $l$ is even,
\begin{align}
R_1&=\sum\limits_{|s|=l+1} \left\{\frac{1}{s!}\right\}^2 Var\left(D^s M_c\left(x+\eta\epsilon_k\right)\epsilon_k^s\right) = \sum\limits_{|s|=l+1} \left\{\frac{1}{s!}\right\}^2 Var\left(D^s M_c\left(x+\eta\epsilon_k\right) \prod\limits_{m=1}^d{\epsilon_{km}^{s_m}}\right) \\
&\lessapprox \sum\limits_{|s|=l+1} \left\{\frac{C}{s!}\right\}^2 Var\left(\prod\limits_{m=1}^d{Z_m^{s_m}}\right) = \sum\limits_{|s|=l+1} \left\{\frac{C}{s!}\right\}^2 \left\{\mathbb{E}\left[\prod\limits_{m=1}^d{Z_m^{2s_m}}\right] - \left(\mathbb{E}\left[\prod\limits_{m=1}^d{Z_m^{s_m}}\right]\right)^2 \right\} \\
&= \sum\limits_{|s|=l+1} \left\{\frac{C}{s!}\right\}^2 \prod\limits_{m=1}^d \frac{\left(2s_m\right)!}{2^{s_m} {s_m}!} \sigma^{2{s_m}},
\end{align}
by Lemma \ref{lemma:bounded_var}. Note that $\mathbb{E}\left[\prod\limits_{m=1}^d{\epsilon_{km}^{s_m}}\right] \approx 0$ so Lemma \ref{lemma:bounded_var} can be applied.

Then, we deal with the fifth term of equation \ref{eq:expand}. Note that triangular inequality, Lemma \ref{lemma:covarvar}, and Lemma \ref{lemma:bounded_var} are used in turn.
\begin{align}
\label{eq:r2_start}
R_2=&\sum\limits_{\substack{|\alpha|=2p-1\\1 \leq p \leq \ceil{l/2}}}\sum\limits_{|\beta|=l+1} \frac{D^\alpha M_c\left(x\right)}{\alpha!} \frac{1}{\beta!} Cov\left(\epsilon_k^\alpha, D^\beta M_c\left(x+\eta\epsilon_k\right) \epsilon_k^\beta\right) \\
&\leq \sum\limits_{\substack{|\alpha|=2p-1\\1 \leq p \leq \ceil{l/2}}}\sum\limits_{|\beta|=l+1} \frac{\left| D^\alpha M_c\left(x\right) \right| }{\alpha!} \frac{1}{\beta!} \left| Cov\left(\epsilon_k^\alpha, D^\beta M_c\left(x+\eta\epsilon_k\right) \epsilon_k^\beta\right) \right| \\
&\leq \sum\limits_{\substack{|\alpha|=2p-1\\1 \leq p \leq \ceil{l/2}}}\sum\limits_{|\beta|=l+1} \frac{\left| D^\alpha M_c\left(x\right) \right| }{\alpha!} \frac{1}{\beta!} \sqrt{Var\left(\epsilon_k^\alpha\right) Var\left(D^\beta M_c\left(x+\eta\epsilon_k\right) \epsilon_k^\beta\right)} \\
&= \sum\limits_{\substack{|\alpha|=2p-1\\1 \leq p \leq \ceil{l/2}}}\sum\limits_{|\beta|=l+1} \frac{\left| D^\alpha M_c\left(x\right) \right| }{\alpha!} \frac{1}{\beta!} \sqrt{Var\left( \prod\limits_{m=1}^d{\epsilon_{km}^{\alpha_m}} \right) Var\left(D^\beta M_c\left(x+\eta\epsilon_k\right) \prod\limits_{m=1}^d{\epsilon_{km}^{\beta_m}}\right)} \\
&\lessapprox \sum\limits_{\substack{|\alpha|=2p-1\\1 \leq p \leq \ceil{l/2}}}\sum\limits_{|\beta|=l+1} \frac{\left| D^\alpha M_c(x) \right| }{\alpha!} \frac{C}{\beta!} \sqrt{Var\left( \prod\limits_{m=1}^d{Z_m^{\alpha_m}} \right) Var\left(\prod\limits_{m=1}^d{Z_m^{\beta_m}}\right)} \\
\label{eq:r2_end}&= \sum\limits_{\substack{|\alpha|=2p-1\\1 \leq p \leq \ceil{l/2}}}\sum\limits_{|\beta|=l+1} \frac{\left| D^\alpha M_c\left(x\right) \right| }{\alpha!} \frac{C}{\beta!} \sqrt{\prod\limits_{m=1}^d \frac{\left(2\alpha_m\right)! \left(2\beta_m\right)!}{2^{\alpha_m+\beta_m} {\alpha_m}! {\beta_m}!} \sigma^{2\left(\alpha_m+\beta_m\right)}}.
\end{align}

The sixth term of equation \ref{eq:expand} is handled in a similar way with Equation \ref{eq:r2_start} - \ref{eq:r2_end}.
\begin{align}
R_3=&\sum\limits_{\substack{|\alpha|=2p\\1 \leq p \leq \floor{l/2}}}\sum\limits_{|\beta|=l+1} \frac{D^\alpha M_c\left(x\right)}{\alpha!} \frac{1}{\beta!} Cov\left(\epsilon_k^\alpha, D^\beta M_c\left(x+\eta\epsilon_k\right) \epsilon_k^\beta\right) \\
&\lessapprox \sum\limits_{\substack{|\alpha|=2p\\1 \leq p \leq \floor{l/2}}}\sum\limits_{|\beta|=l+1} \frac{\left| D^\alpha M_c\left(x\right) \right| }{\alpha!} \frac{C}{\beta!} \sqrt{Var\left( \prod\limits_{m=1}^d{Z_m^{\alpha_m}} \right) Var\left(\prod\limits_{m=1}^d{Z_m^{\beta_m}}\right)} \\
&= \sum\limits_{\substack{|\alpha|=2p\\1 \leq p \leq \floor{l/2} \\ \forall s_m\textrm{ is zero or even}}}\sum\limits_{|\beta|=l+1} \frac{\left| D^\alpha M_c\left(x\right) \right| }{\alpha!} \frac{C}{\beta!} \sqrt{Var\left( \prod\limits_{m=1}^d{Z_m^{\alpha_m}} \right) Var\left(\prod\limits_{m=1}^d{Z_m^{\beta_m}}\right)} \notag \\
&+ \sum\limits_{\substack{|\alpha|=2p\\1 \leq p \leq \floor{l/2} \\ \exists s_m\textrm{ is odd}}}\sum\limits_{|\beta|=l+1} \frac{\left| D^\alpha M_c\left(x\right) \right| }{\alpha!} \frac{C}{\beta!} \sqrt{Var\left( \prod\limits_{m=1}^d{Z_m^{\alpha_m}} \right) Var\left(\prod\limits_{m=1}^d{Z_m^{\beta_m}}\right)} \\
&= \sum\limits_{\substack{|\alpha|=2p\\1 \leq p \leq \floor{l/2} \\ \forall s_m\textrm{ is zero or even}}}\sum\limits_{|\beta|=l+1} \frac{\left| D^\alpha M_c\left(x\right) \right| }{\alpha!} \frac{C}{\beta!} \notag \sqrt{\left[\prod\limits_{m=1}^d \frac{\left(2\alpha_m\right)!}{2^{\alpha_m} {\alpha_m}!} \sigma^{2{\alpha_m}}  - \prod\limits_{m=1}^d \frac{\left(\alpha_m!\right)^2}{2^{\alpha_m} \left\{\left(\alpha_m/2\right)!\right\}^2} \sigma^{{2\alpha_m}}\right] \prod\limits_{m=1}^d \frac{\left(2\beta_m\right)!}{2^{\beta_m} {\beta_m}!} \sigma^{2{\beta_m}}} \notag \\
&+ \sum\limits_{\substack{|\alpha|=2p\\1 \leq p \leq \floor{l/2} \\ \exists s_m\textrm{ is odd}}}\sum\limits_{|\beta|=l+1} \frac{\left| D^\alpha M_c\left(x\right) \right| }{\alpha!} \frac{C}{\beta!} \sqrt{\prod\limits_{m=1}^d \frac{\left(2\alpha_m\right)!}{2^{\alpha_m} {\alpha_m}!} \sigma^{2{\alpha_m}} \prod\limits_{m=1}^d \frac{\left(2\beta_m\right)!}{2^{\beta_m} {\beta_m}!} \sigma^{2{\beta_m}}}.
\end{align}

As a result,
\begin{align}
\therefore \widetilde{M}_c(x) \approx &\sum\limits_{\substack{|s|=2p-1\\1 \leq p \leq \ceil{l/2}}} \left\{\frac{D^s M_c\left(x\right)}{s!}\right\}^2 \prod\limits_{m=1}^d \frac{\left(2s_m\right)!}{2^{s_m} {s_m}!} \sigma^{2{s_m}} \notag \\
&+ \sum\limits_{\substack{|s|=2p\\1 \leq p \leq \floor{l/2} \\ \forall s_m\textrm{ is zero or even}}} \left\{\frac{D^s M_c\left(x\right)}{s!}\right\}^2 \left[\prod\limits_{m=1}^d \frac{\left(2s_m\right)!}{2^{s_m} {s_m}!} \sigma^{2{s_m}}  - \prod\limits_{m=1}^d \frac{\left(s_m!\right)^2}{2^{s_m} \left\{\left(s_m/2\right)!\right\}^2} \sigma^{{2s_m}}\right] \notag \\
&+ \sum\limits_{\substack{|s|=2p\\1 \leq p \leq \floor{l/2} \\ \exists s_m\textrm{ is odd}}} \left\{\frac{D^s M_c\left(x\right)}{s!}\right\}^2 \prod\limits_{m=1}^d \frac{\left(2s_m\right)!}{2^{s_m} {s_m}!} \sigma^{2{s_m}} + R_1 + 2R_2 + 2R_3,
\end{align}
\begin{align}
&|R_1| \lessapprox \sum\limits_{|s|=l+1} \left\{\frac{C}{s!}\right\}^2 \prod\limits_{m=1}^d \frac{\left(2s_m\right)!}{2^{s_m} {s_m}!} \sigma^{2{s_m}}, \\
&|R_2| \lessapprox \sum\limits_{\substack{|\alpha|=2p-1\\1 \leq p \leq \ceil{l/2}}}\sum\limits_{|\beta|=l+1} \frac{\left| D^\alpha M_c\left(x\right) \right| }{\alpha!} \frac{C}{\beta!} \sqrt{\prod\limits_{m=1}^d \frac{\left(2\alpha_m\right)! \left(2\beta_m\right)!}{2^{\alpha_m+\beta_m} {\alpha_m}! {\beta_m}!} \sigma^{2\left(\alpha_m+\beta_m\right)}},
\end{align}
\begin{align}
|R_3| \lessapprox &\sum\limits_{\substack{|\alpha|=2p\\1 \leq p \leq \floor{l/2} \\ \forall s_m\textrm{ is zero or even}}}\sum\limits_{|\beta|=l+1} \frac{\left| D^\alpha M_c\left(x\right) \right| }{\alpha!} \frac{C}{\beta!} \notag \sqrt{\left[\prod\limits_{m=1}^d \frac{\left(2\alpha_m\right)!}{2^{\alpha_m} {\alpha_m}!} \sigma^{2{\alpha_m}}  - \prod\limits_{m=1}^d \frac{\left(\alpha_m!\right)^2}{2^{\alpha_m} \left\{\left(\alpha_m/2\right)!\right\}^2} \sigma^{{2\alpha_m}}\right] \prod\limits_{m=1}^d \frac{\left(2\beta_m\right)!}{2^{\beta_m} {\beta_m}!} \sigma^{2{\beta_m}}} \notag \\
&+ \sum\limits_{\substack{|\alpha|=2p\\1 \leq p \leq \floor{l/2} \\ \exists s_m\textrm{ is odd}}}\sum\limits_{|\beta|=l+1} \frac{\left| D^\alpha M_c\left(x\right) \right| }{\alpha!} \frac{C}{\beta!} \sqrt{\prod\limits_{m=1}^d \frac{\left(2\alpha_m\right)!}{2^{\alpha_m} {\alpha_m}!} \sigma^{2{\alpha_m}} \prod\limits_{m=1}^d \frac{\left(2\beta_m\right)!}{2^{\beta_m} {\beta_m}!} \sigma^{2{\beta_m}}}.
\end{align}
\end{proof}

\bibliographyNew{appendix_bio}
\bibliographystyleNew{icml2018_appendix}



\end{document}